\documentclass{article}
\usepackage{ijcai24}

\usepackage{times}
\usepackage{soul}
\usepackage{url}
\usepackage[hidelinks]{hyperref}
\usepackage[utf8]{inputenc}
\usepackage[small]{caption}
\usepackage{graphicx}
\usepackage{amsmath}
\usepackage{amsthm}
\usepackage{booktabs}
\usepackage{algorithm}
\usepackage{algorithmic}
\usepackage[switch]{lineno}

\usepackage{subcaption}

\usepackage{stmaryrd}

\usepackage{dsfont}

\usepackage{graphicx,calc}
\newlength\myheight
\newlength\mydepth
\settototalheight\myheight{Xygp}
\newcommand*\inlinegraphics[1]{%
  \settototalheight\myheight{Xygp}%
  \settodepth\mydepth{Xygp}%
  \raisebox{-\mydepth}{\includegraphics[height=\myheight]{#1}}%
}

\usepackage{amsfonts}
\usepackage{amssymb}
\usepackage{mathtools}

\usepackage{mathabx}

\newcommand{\ie}{\textit{i}.\textit{e}., }
\newcommand{\eg}{\textit{e}.\textit{g}.\ }

\DeclareMathOperator*{\argmax}{argmax}

\newtheorem{example}{Example}

\newtheorem{proposition}{Proposition}

\theoremstyle{definition}
\newtheorem{definition}{Definition}
\newtheorem{obs}{Observation}

\theoremstyle{remark}
\newtheorem{remark}{Remark}

\title{Improving Neural-based Classification with Logical Background Knowledge}

\author{
Arthur Ledaguenel$^{1,2}$
\and
Céline Hudelot$^2$\and
Mostepha Khouadjia$^1$\\
\affiliations
$^1$IRT SystemX\\
$^2$MICS, CentraleSupélec\\
\emails
arthur.ledaguenel@irt-systemx.fr
}

\begin{document}

\maketitle

\begin{abstract}
Neurosymbolic AI is a growing field of research aiming to combine neural networks learning capabilities with the reasoning abilities of symbolic systems. This hybridization can take many shapes. In this paper, we propose a new formalism for supervised multi-label classification with propositional background knowledge. We introduce a new neurosymbolic technique called semantic conditioning at inference, which only constrains the system during inference while leaving the training unaffected. We discuss its theoritical and practical advantages over two other popular neurosymbolic techniques: semantic conditioning and semantic regularization. We develop a new multi-scale methodology to evaluate how the benefits of a neurosymbolic technique evolve with the scale of the network. We then evaluate experimentally and compare the benefits of all three techniques across model scales on several datasets. Our results demonstrate that semantic conditioning at inference can be used to build more accurate neural-based systems with fewer resources while guaranteeing the semantic consistency of outputs.
\end{abstract}

\section*{Introduction}

Neurosymbolic AI is a growing field of research aiming to combine neural network learning capabilities with the reasoning abilities of symbolic systems. This hybridization can take many shapes depending on how the neural and symbolic components interact, like shown in \cite{Kautz2022T}.

An important sub-field of neurosymbolic AI is Informed Machine Learning \cite{VonRueden2023}, which studies how to leverage background knowledge to improve neural-based systems. There again, proposed techniques in the literature can be of very different nature depending on the type of task (\eg regression, classification, detection, generation, etc.), the formalism used to represent the background knowledge (\eg mathematical equations, knowledge graphs, logics, etc.), the stage at which knowledge is embedded (\eg data processing, neural architecture design, learning procedure, inference procedure, etc.) and benefits expected from the hybridization (\eg explainability, performance, frugality, etc.).

The contributions and outline of the paper are the following. After preliminary notions in Section \ref{sec:prem}, we introduce in Section \ref{sec:formalism} our new formalism to represent a task of supervised multi-label classification with propositional background knowledge and to describe neural-based classification systems. Then we build upon this formalism in Section \ref{sec:sci} to re-frame existing neurosymbolic techniques, define ours, semantic conditioning at inference, and compare their properties, benefits and implementations. In Section \ref{sec:experiments}, we present our multi-scale evaluation methodology and compare the results for all techniques on three different tasks. We discuss related works in Section \ref{sec:related} and conclude with possible future research questions in Section \ref{sec:conclusion}. Proofs for all stated propositions can be found in the supplementary materials.

\section{Preliminaries} \label{sec:prem}
\subsection{Propositional Logic}
A \textbf{propositional signature} is a set $\mathcal{S}$ of symbols called \textbf{variables} (\eg $\mathcal{S} = \{a,b\}$). A \textbf{propositional formula} is formed inductively from variables and other formulas by using unary ($\neg$, which expresses negation) or binary ($\lor, \land$, which express disjunction and conjunction respectively) connectives (\eg $\kappa = a \land b$ which is \textit{true} if both variables $a$ and $b$ are \textit{true}). A \textbf{model} is an application $\nu:\mathcal{S} \mapsto \{0, 1\}$. A model can be inductively extended to define a \textbf{valuation} $\nu^*$ on all formulas using the standard semantics of propositional logic (\eg $\nu^*(a \land b) = \nu(a) \times \nu(b)$). We say that a model $\nu$ \textbf{entails} a formula $\kappa$, noted $\nu \models \kappa$, if $\nu^*(\kappa) = 1$. We say that a formula is \textbf{satisfiable} when it is entailed by at least one model. We use the symbol $\top$ to represent \textbf{tautologies} (\ie formulas which are entailed by all models). Two formulas are said \textbf{equivalent}, noted $\kappa \equiv \gamma$, if they are entailed by exactly the same models. We refer to \cite{Russell2021} for more details on propositional logic.

\subsection{Distributions}
One challenge of neurosymbolic AI is to bridge the gap between the discrete nature of logic and the continuous nature of neural networks. In this section, we define distributions and introduce concepts from probabilistic propositional logic that will provide the interface between these two realms.

A \textbf{distribution} on \textbf{binary variables} $\{Y_j\}_{1\leq j \leq k}$ is an application $\mathbf{E}:\{0, 1\}^k \mapsto \mathbb{R}^+$. The \textbf{null distribution} is the application that maps all states of $\{0, 1\}^k$ to $0$. The \textbf{partition function} $\mathsf{Z}:\mathbf{E} \mapsto \sum_{\mathbf{y} \in \{0, 1\}^k} \mathbf{E}(\mathbf{y})$ maps each distribution to its sum, and we note $\overline{\mathbf{E}} := \frac{\mathbf{E}}{\mathsf{Z}(\mathbf{E})}$ the normalized distribution (when $\mathbf{E}$ is non-null). In the remaining of the paper, we will use $\mathbf{P}$ to indicate probability (\ie normalized) distributions. The \textbf{mode} of a distribution $\mathbf{E}$ is its most probable state, \ie $\underset{\mathbf{y} \in \{0, 1\}^k}{\argmax}\mathbf{E}(\mathbf{y})$.

A standard distribution in machine learning is the exponential distribution.
\begin{definition}
    Given an activation vector $\mathbf{a} \in \mathbb{R}^k$, we define the \textbf{exponential distribution} as:
    \begin{equation*}
    \mathbf{E}(\cdot | \mathbf{a}):  \mathbf{y} \mapsto \prod_{1 \leq i\leq k} e^{a_i.y_i}
    \end{equation*} 
    We will also note $\mathbf{P}(\cdot | \mathbf{a}) = \overline{\mathbf{E}(\cdot | \mathbf{a})}$ the corresponding normalized probability distribution.
\end{definition}

\begin{remark}
The exponential probability distribution is the joint distribution of independent Bernouilli variables $\mathcal{B}(p_i)_{1 \leq i \leq k}$ with $p_i = \mathsf{s}(a_i)$, where $\mathsf{s}(\mathbf{a}) = (\frac{e^{a_j}}{1 + e^{a_j}})_{1\leq j \leq k}$ is the sigmoid function.
\end{remark}

\begin{definition}
The \textbf{semantic projection} of a distribution $\mathbf{E}$ on a propositional formula $\kappa$, is the distribution:
        \begin{equation}
            \mathbf{E}(\cdot \land \kappa): \mathbf{y} \to \mathbf{E} . \mathds{1}[\mathbf{y} \models \kappa]
        \end{equation}
This operation consists of carving out invalid states (\eg states that do not entail the propositional formula $\kappa$) from the distribution.

The \textbf{probability} of a \textbf{propositional formula} $\kappa$ under a probability distribution $\mathbf{P}$ is given by:
    \begin{equation}
        \mathbf{P}(\kappa) := \mathsf{Z}(\mathbf{P}(\cdot \land \kappa)) = \sum_{\mathbf{y}} \mathbf{P}(\mathbf{y}\land \kappa) = \sum_{\mathbf{y} \models \kappa} \mathbf{P}(\mathbf{y})
    \end{equation}

\begin{remark}
    Since $\mathbf{P}(\cdot | \mathbf{a})$ is strictly positive (for all $\mathbf{a}$), if $\kappa$ is satisfiable, then $\mathbf{P}(\kappa | \mathbf{a}) > 0$.
\end{remark}
\end{definition}

\begin{definition}
    The \textbf{semantic conditioning} of a probability distribution $\mathbf{P}$ by a \textbf{satisfiable} propositional formula $\kappa$ gives:
\begin{equation}
    \mathbf{P}(\cdot | \kappa): \mathbf{y} \to \overline{\mathbf{P}(\mathbf{y} \land \kappa)}
\end{equation}
\end{definition}

\begin{remark}
    When using projection and conditioning on an exponential distribution, we note:
    \begin{gather}
        \mathbf{P}(\kappa | \mathbf{a}):=\mathsf{Z}(\mathbf{P}(\cdot | \mathbf{a}).\mathds{1}[\mathbf{y} \models \kappa]) \\
        \mathbf{P}(\cdot | \mathbf{a}, \kappa):=\frac{\mathbf{P}(\cdot | \mathbf{a}).\mathds{1}[\mathbf{y} \models \kappa]}{\mathbf{P}(\kappa | \mathbf{a})}
    \end{gather}
    Since $\mathbf{P}(\cdot | \mathbf{a})$ is strictly positive (for all $\mathbf{a} \in \mathbb{R}^k$), if $\kappa$ is satisfiable, then $\mathbf{P}(\kappa | \mathbf{a}) > 0$.
    Computing $\mathbf{P}(\kappa | \mathbf{a})$ is a \textbf{counting} problem called Probabilistic Query Estimation (PQE). Computing the mode of $\mathbf{P}(\cdot | \mathbf{a}, \kappa)$ is an \textbf{optimization} problem called Maximum a posteriori (MAP) estimation. Solving these problems is at the core of several neurosymbolic techniques that we will introduce later.
\end{remark}

\section{Multi-label classification with background knowledge} \label{sec:formalism}
In this section, we define a formalism for \textbf{multi-label classification with background knowledge} that enables to explicitly express the structure of the output space. The objective is to model the relationship between an input domain $\mathcal{X}:=\mathbb{R}^d$ and a discrete output domain $\mathcal{Y}:=\{0, 1\}^k$ on which we have background knowledge. This background knowledge is expressed as a \textbf{satisfiable} propositional formula $\kappa$ over the signature $\mathcal{S}:= \{Y_j\}_{1\leq j \leq k}$ (one symbol per output variable). There is a direct mapping between an element of the output domain $\mathbf{y} \in \{0, 1\}^k$ and the model of $\mathcal{S}$ that maps each variable $Y_j$ to the binary value $y_j$. Therefore, we will note abusively $\mathbf{y} \models \kappa$ to denote that the model corresponding to $\mathbf{y}$ entails the propositional formula $\kappa$. Hence, $\kappa$ determines which outputs $\mathbf{y} \in \mathcal{Y}$ are consistent with our background knowledge on the task (\ie $\mathbf{y} \models \kappa$).

A dataset for such a task is $\mathcal{D}=(x^i, \mathbf{y}^i)_{1\leq i \leq n}$ with input samples $x^i \in \mathbb{R}^d$ and labels $\mathbf{y}^i \in \{0, 1\}^k$, such that all labels entail the background knowledge (\ie $\forall 1\leq i \leq n, \mathbf{y}^i \models \kappa$). In this definition, datasets are assumed to provide \textbf{full supervision} on all classes and to always be \textbf{consistent}. However, some techniques allow for a relaxation of both assumptions, enabling to use imperfect background knowledge and weakly or semi-supervised datasets.

In deep learning, we often adopt a functional framework in the sense that we assume a functional dependency between input and output domains and design the neural network to model that dependency. A differentiable cost function is used to measure the distance between the predictions and the labels, and the weights of the network are optimized using backpropagation to minimize the empirical cost. However, in case of classification, such a framework cannot be applied strictly: since the output space is discrete, a differentiable cost function cannot be defined directly on it.

Hence, we adopt a slightly modified framework, that we here call pseudo-functional, where a third module (besides the neural and loss modules) has to be defined :
\begin{definition} \label{def:neural_system}
A \textbf{neural-based classification system} (see Figure \ref{fig:class-syst}) for multi-label classification is the given of :
\begin{itemize}
    \item a \textbf{parametric differentiable} (\ie neural) module $\mathsf{M}$, called the \textbf{model}, which takes as inputs $x \in \mathbb{R}^d$, parameters $\theta \in \Theta$ and outputs $\mathsf{M}(x, \theta) := \mathsf{M}_{\theta}(x):=\mathbf{a} \in \mathbb{R}^k$, called the \textbf{activation vector}.
    \item a \textbf{non-parametric differentiable} module $\mathsf{L}$, called the \textbf{loss} module, which takes $\mathbf{a} \in \mathbb{R}^k$ and $\mathbf{y} \in \{0, 1\}^k$ as inputs and outputs a scalar.
    \item a \textbf{non-parametric} module $\mathsf{I}$, called the \textbf{inference} module, which takes $\mathbf{a} \in \mathbb{R}^k$ as input and outputs a prediction $\hat{\mathbf{y}} \in \{0, 1\}^k$.
\end{itemize}
\end{definition}

\begin{figure}[t]
\centering
\includegraphics[width=\linewidth]{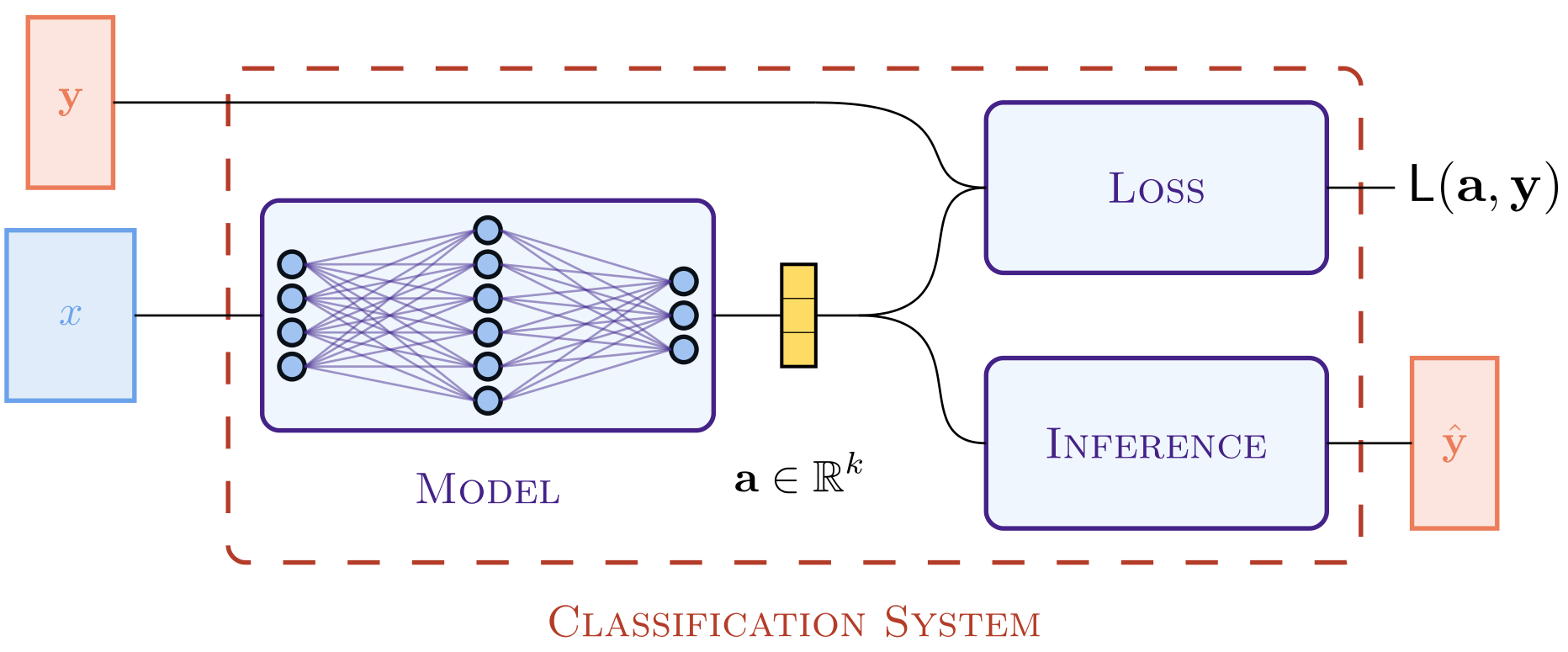}
\caption{Illustration of a neural-based classification system}
\label{fig:class-syst}
\end{figure}

This paper does not discuss the architecture of the neural model (\eg fully connected, convolutional, transformer-based, etc.) which mainly depends on the modality of the input space (\eg images, texts, etc.), but rather focuses on the two other modules, to embed the structural prior we have on the output space (\ie the background knowledge). We give a few examples of standard classification structures that can be explicitly specified using the background knowledge and their corresponding loss and inference modules.

\begin{example} \label{ex:independent}
\textbf{Independent multi-label} classification (\textit{imc}) corresponds to $\kappa=\top$: since all variables are independent, every combination is semantically valid \footnote{The symbol $\perp$ is used to represent the independence assumption and and has nothing to do with its logical meaning to designate a fallacy (\ie a formula that can't be satisfied).}. The standard choice in this case is to apply a sigmoid layer on activation scores to turn them into probability scores. The loss is the binary cross-entropy between probability scores and labels, and a variable is predicted to be \textit{true} if its probability is above $0.5$ (or equivalently its activation score is above $0$). This results in the following modules:
\begin{equation}
    \begin{multlined}
        \mathsf{L}_{imc}(\mathbf{a}, \mathbf{y}) := \mathtt{BCE}(\mathsf{s}(\mathbf{a}), \mathbf{y}) \\
         = - \sum_j y_j.\log(\mathsf{s}(a_j)) + (1-y_j).\log(1-\mathsf{s}(a_j))
    \end{multlined}
\end{equation}
\begin{equation}
    \mathsf{I}_{imc}(\mathbf{a}) := \mathbf{1}[\mathbf{a} \geq 0]
\end{equation}
where $\mathtt{BCE}$ is the binary cross-entropy and $\mathbf{1}[z]:= \left\{
    \begin{array}{ll}
        1 & \mbox{if z true} \\
        0 & \mbox{otherwise}
    \end{array}
\right.$ the indicator function.
\end{example}

\begin{example}
\textbf{Categorical} classification arises when one and only one output variable is \textit{true} for a given input sample (\eg mapping an image to a single digit in $\llbracket 0, 9 \rrbracket$ for MNIST). These constraints can easily be enforced by the following propositional formula:
 \begin{equation}
     \kappa_{\odot_k} := \bigg(\bigvee_{1 \leq j \leq k} Y_j \bigg) \land \bigg(\bigwedge_{1 \leq j < l \leq k} (\neg Y_j \lor \neg Y_l) \bigg)
 \end{equation}
where the first part ensures that at least one variable is \textit{true} and the second part prevents two variables to be \textit{true} simultaneously.
For categorical classification, the sigmoid layer is replaced by a softmax layer and the variable with the maximum score is predicted, which leads to the following modules:
\begin{equation}
    \mathsf{L}_{\odot_k}(\mathbf{a}, \mathbf{y}) := \mathtt{CE}(\mathtt{s}(\mathbf{a}), \mathbf{y}) = - \log(\langle \sigma(\mathbf{a}), \odot_k(j) \rangle)
\end{equation}
\begin{equation}
    \mathsf{I}_{\odot_k}(\mathbf{a}) := \odot_k(\argmax(\mathbf{a}))
\end{equation}
where $\mathtt{CE}$ is the cross-entropy, $\sigma(\mathbf{a}) = (\frac{e^{a_j}}{\sum_l e^{a_l}})_{1\leq j \leq k}$ and $\odot_k$ gives the one-hot encoding (starting at 1) of $j \in \llbracket 1, k \rrbracket$, \eg $\odot_4(2)=(0, 1, 0, 0$).
\end{example}

\begin{example}
\textbf{Hierarchical} classification on a set of variables $\{Y_j\}_{1\leq j \leq k}$ is usually formulated with a directed acyclic graph $G=(Y, E_h)$ where the nodes are the variables and the edges $E_h$ express subsumption between those variables (\eg a dog is an animal). This formalism can even be enriched with exclusion edges $H=(Y, E_h, E_e)$ (\eg an input cannot be both a dog and a cat), like in HEX-graphs \cite{Deng2014} or AND-OR graphs \cite{Aleksander86}. There again, the translation to propositional logic is straightforward:
\begin{equation}
    \kappa_H := \bigg(\bigwedge_{(i, j) \in E_h} Y_i \lor \neg Y_j \bigg) \land \bigg(\bigwedge_{(i, j) \in E_e} (\neg Y_i \lor \neg Y_j) \bigg)
\end{equation}
where the first part ensures that a son node cannot be \textit{true} if its father node is not and the second part prevents two mutually exclusive nodes to be \textit{true} simultaneously.
In this case, there is no consensus on what modules to use for the neural-based classification system: \cite{Muller2020AHL} introduces a hierarchical loss to penalize more errors on higher classes of the hierarchy while \cite{giunchiglia_coherent_2020,Deng2014} impact both loss and inference.
\end{example}

Moreover, propositional background knowledge can be used to define very diverse valid output spaces (\eg Sudoku solutions \cite{Augustine2022VisualSP}, simple paths in a graph \cite{Xu2018,Ahmed2022spl}, preference rankings \cite{Xu2018}, etc.).

\section{Unifying neurosymbolic techniques} \label{sec:sci}
The purpose of a neurosymbolic technique is to automatically derive appropriate loss and inference modules from the propositional background knowledge, therefore generalizing the work made on independent, categorical and hierarchical cases to arbitrary structures.

We build upon our formalism to re-frame two standard neurosymbolic techniques in the literature (semantic conditioning and semantic regularization) and define our proposed technique, named \textbf{semantic conditioning at inference}.

\subsection{Semantic conditioning}
\begin{definition}
A neural-based classification system is performing \textbf{semantic conditioning} (\textit{sc}) on $\kappa$ if its loss and inference modules are:
    \begin{gather}
        \mathsf{L}_{|\kappa}(\mathbf{a}, \mathbf{y}) = - \log(\mathbf{P}(\mathbf{y} | \mathbf{a}, \kappa)) \\
        \mathsf{I}_{|\kappa}(\mathbf{a}) = \underset{\mathbf{y} \in \{0, 1\}^k}{\argmax}\mathbf{P}(\mathbf{y} | \mathbf{a}, \kappa)
    \end{gather}
\end{definition}

The generality of \textit{sc} is best illustrated by showing that standard modules for independent and categorical classification are particular cases of semantic conditioning on their respective background knowledge:
\begin{proposition} \label{prop:sem_cond_equiv}
    \begin{equation}
        \begin{split}
        \mathsf{L}_{|\top}(\mathbf{a}, \mathbf{y}) = \mathsf{L}_{imc}(\mathbf{a}, \mathbf{y}) \quad & \mathsf{I}_{|\top}(\mathbf{a}) = \mathsf{I}_{imc}(\mathbf{a}) \\
        \mathsf{L}_{|\kappa_{\odot_k}}(\mathbf{a}, \mathbf{y}) = \mathsf{L}_{\odot_k}(\mathbf{a}, \mathbf{y}) \quad & \mathsf{I}_{|\kappa_{\odot_k}}(\mathbf{a}) = \mathsf{I}_{\odot_k}(\mathbf{a})
        \end{split}
    \end{equation}
\end{proposition}

Semantic conditioning was introduced for hierarchical knowledge in \cite{Deng2014} and \cite{Ahmed2022spl} generalized it to arbitrary formulas in conjunctive normal form.

\subsection{Semantic regularization}
\begin{definition}
A neural-based classification system is performing \textbf{semantic regularization} (\textit{sr}) on $\kappa$ if its loss and inference modules are:
    \begin{gather}
        \mathsf{L}_{r(\lambda)}(\mathbf{a}, \mathbf{y}) = \mathsf{L}_{imc}(\mathbf{a}, \mathbf{y}) - \lambda.\log(\mathbf{P}(\kappa | \mathbf{a})) \\
        \mathsf{I}_{imc}(\mathbf{a}) = \underset{\mathbf{y} \in \{0, 1\}^k}{\argmax}\mathbf{P}(\mathbf{y} | \mathbf{a})
    \end{gather}
\end{definition}

This regularization term is equivalent to the semantic loss first introduced in \cite{Xu2018}.

\subsection{Semantic conditioning at inference}
Semantic conditioning at inference is derived from semantic conditioning, but only applies conditioning to the inference module (\ie infers the most probable output that entails $\kappa$).

\begin{definition}
A neural-based classification system is performing \textbf{semantic conditioning at inference} (\textit{sci}) on $\kappa$ if its loss and inference modules are :
    \begin{gather}
        \mathsf{L}_{imc}(\mathbf{a}, \mathbf{y}) = - \log(\mathbf{P}(\mathbf{y} | \mathbf{a})) \\
        \mathsf{I}_{|\kappa}(\mathbf{a}) = \underset{\mathbf{y} \in \{0, 1\}^k}{\argmax}\mathbf{P}(\mathbf{y} | \mathbf{a}, \kappa)
    \end{gather}
\end{definition}

Semantic conditioning at inference retains two key properties of semantic conditioning: \textbf{syntactic invariance} (\ie conditioning on equivalent formulas gives identical results) and \textbf{semantic consistency} at inference (\ie all inferred outputs will entail the background knowledge, untrue of semantic regularization techniques).
\begin{proposition}[Syntactic invariance] \label{prop:syntactic_invariance}
\begin{equation*}
    \kappa \equiv \gamma \implies \mathsf{I}_{|\kappa} = \mathsf{I}_{|\gamma}
\end{equation*}
\end{proposition}

\begin{proposition}[Semantic consistency] \label{prop:sem_consistency}
\begin{equation*}
    \forall \mathbf{a}, \mathsf{I}_{|\kappa}(\mathbf{a}) \models \kappa
\end{equation*}
\end{proposition}

Besides, when performing inference based on identical model modules and learned parameters, \textit{sci} \textbf{guarantees} greater or equal accuracy compared to traditional \textit{imc} inference (\ie if $\mathsf{I}_{imc}$ infers the right labels, then $\mathsf{I}_{|\kappa}$ will also infer the right labels): 
\begin{proposition} \label{prop:acc_guarantee}
\begin{equation*} 
\begin{multlined}
    \forall (x, \mathbf{y}) \in \mathcal{D}, \forall \theta \in \Theta, \\
    \mathsf{I}_{imc}(\mathsf{M}_{\theta}(x))=\mathbf{y} \implies \mathsf{I}_{|\kappa}(\mathsf{M}_{\theta}(x))=\mathbf{y}
\end{multlined}
\end{equation*} 
\end{proposition}

Semantic conditioning at inference has other impactful properties that make it a suitable choice compared to semantic conditioning and regularization. First, whereas \textit{sc} and \textit{sr} rely on PQE to compute their loss module, \textit{sci} only relies on MAP estimation for its inference module. We mention below why this has huge consequences on computational tractability and efficiency. Second, training with the standard independent multi-label classification loss leads to \textbf{independent} representations of the different variables. This enables to change background knowledge after training, for instance if the background knowledge is unavailable at training time or susceptible to evolve, which is not possible when learned representations are tightly entangled through the knowledge. This is a particularly useful property in the era of \textbf{foundation models} \cite{Bommasani21}, which are first pre-trained on massive amounts of general data to then be fine-tuned on a multitude of heterogeneous downstream tasks with different background knowledge. Finally, in a scientific perspective, it seems interesting to study the efficiency of loss and inference modules separately, to understand how much each component contributes to overall performance gains.

\subsection{Complexity and implementations} \label{sec:complexity}
Computing $\mathsf{L}_{r(\lambda)}(\mathbf{a}, \mathbf{y})$ (for \textit{sr}) and $\mathsf{L}_{|\kappa}(\mathbf{a}, \mathbf{y})$ (for \textit{sc}) rely on PQE, while computing $\mathsf{I}_{|\kappa}(\mathbf{a})$ (for \textit{sc} and \textit{sci}) relies on MAP estimation. In the general case, PQE and MPE are both $\mathtt{NP}$-hard (by reduction of the SAT problem) \cite{Suciu2020}. Hence, all three neurosymbolic techniques quickly become intractable for a large number of variables in the general case.

However, for propositional formulas of \textbf{bounded tree-width}, both PQE and MPE can be computed in time \textbf{linear in the number of variables} \cite{Darwiche2009} (and exponential in the tree-width). In our experiments for instance, these computations bring no significant overhead on top of the neural module.

Besides, counting problems are known to be much harder in general than optimization problems \cite{toda_pp_1991}, which reflects in the greater development and efficiency of optimization solvers. Since \textit{sc} and \textit{sr} rely on PQE to compute their loss module, whereas \textit{sci} only relies on MAP estimation for its inference module, there are cases where \textit{sci} is tractable while \textit{sc} and \textit{sr} are not. This property is exploited in \cite{niepert_implicit_2021} to approximate PQE using the Gumbel-max trick \cite{papandreou_perturb-and-map_2011}.


Implementation wise, we encountered two methods in the literature. The first was introduced in \cite{Deng2014}: it represents the conditioned distribution as a Conditional Random Field, uses a custom compilation algorithm to convert the propositional formula into a minimal junction tree and then applies a sparse message passing procedure. However, the compilation technique was specifically designed to process hierarchical knowledge expressed as a HEX-graph (see Section \ref{sec:formalism}). More recently, the neurosymbolic community has mainly been using knowledge compilation: \cite{Xu2018} uses Sentential Decision Diagrams \cite{Darwiche2011} to compute their semantic loss, while \cite{Ahmed2022spl} uses tractable circuits to perform both PQE and MAP for semantic conditioning. For our experiments, we implemented our own version of \cite{Deng2014} algorithm (their code was not publicly available) and tested knowledge compilation based techniques with no noticeable difference in compute time.
\footnote{Our GitHub will be made publicly available in the final version}.

\section{Experiments} \label{sec:experiments}
We showed in the previous section how \textit{sci} was provably more accurate than \textit{imc}. In this section, we illustrate to what extent this guarantee translates in practice. We also compare \textit{sci} to existing techniques \textit{sc} and \textit{sr}.



\begin{figure*}[t]
\begin{subfigure}{0.32\textwidth}
    \centering
    \includegraphics[width=\textwidth]{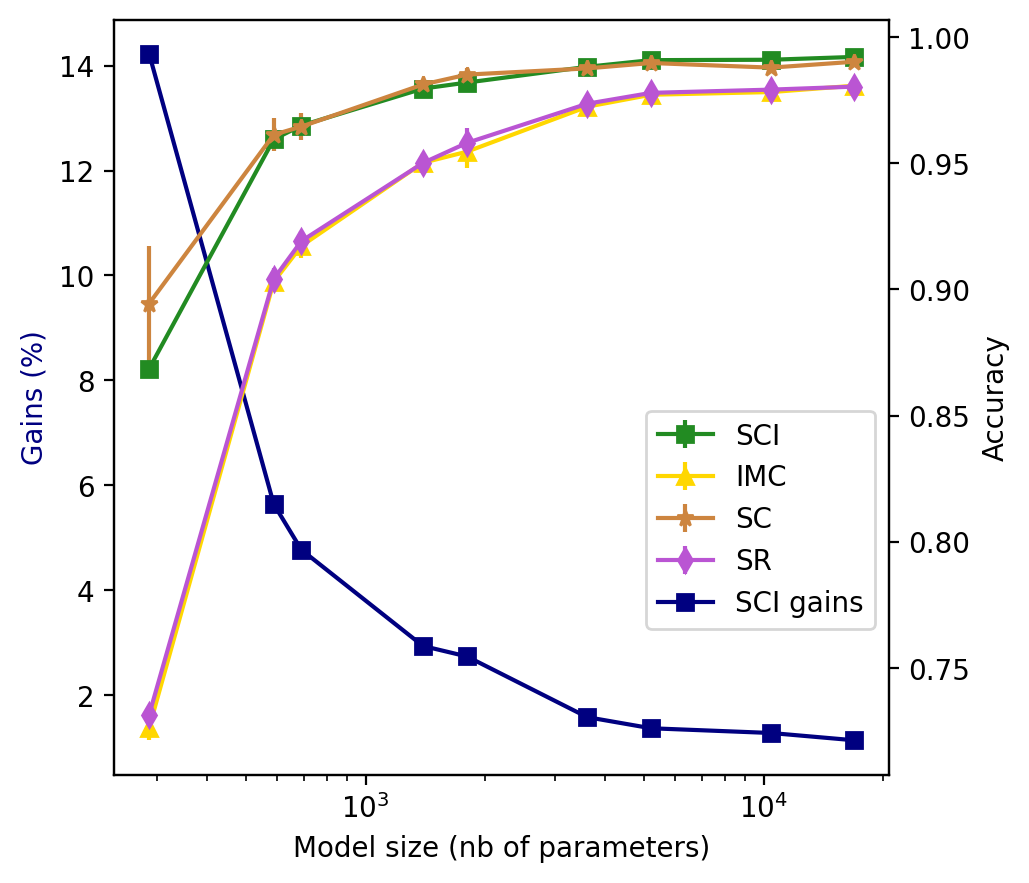}
    \caption{MNIST after 100 epochs}
    \label{fig:mnist_modelsize}
\end{subfigure}
\begin{subfigure}{0.32\textwidth}
    \centering
    \includegraphics[width=\textwidth]{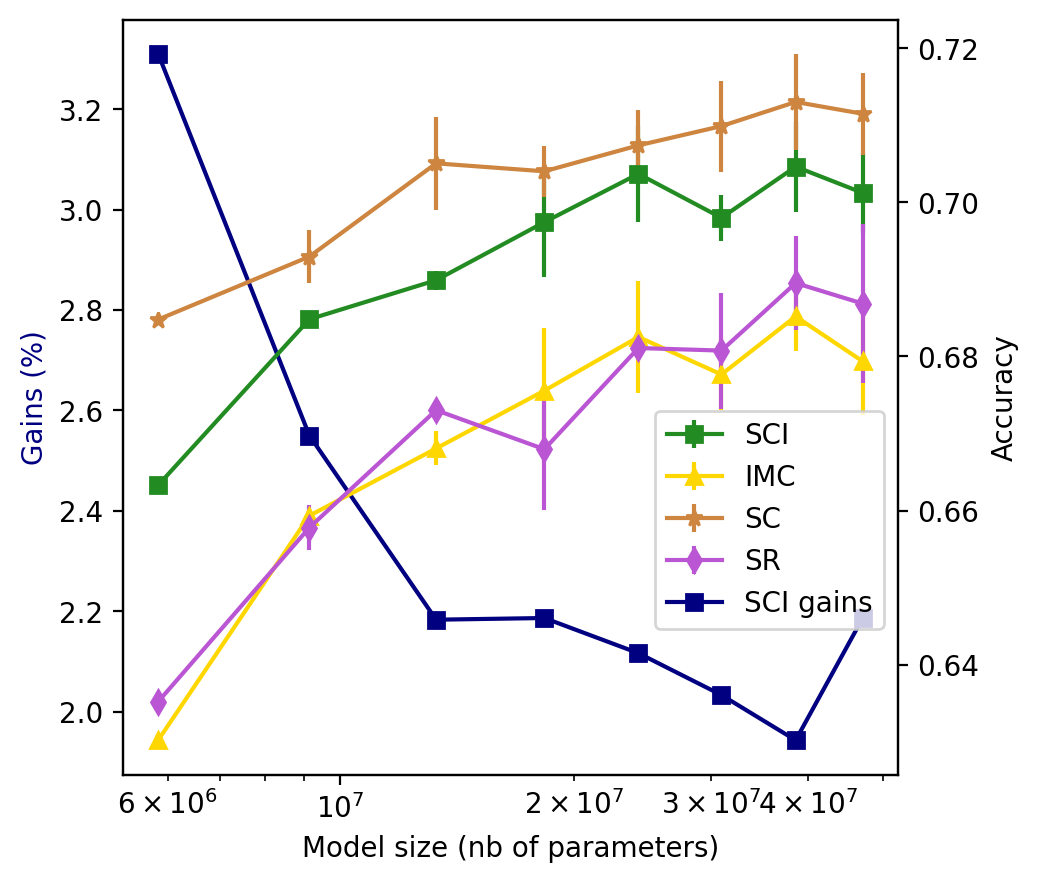}
    \caption{Cifar after 100 epochs}
    \label{fig:cifar_modelsize}
\end{subfigure}
\begin{subfigure}{0.32\textwidth}
    \centering
    \includegraphics[width=\textwidth]{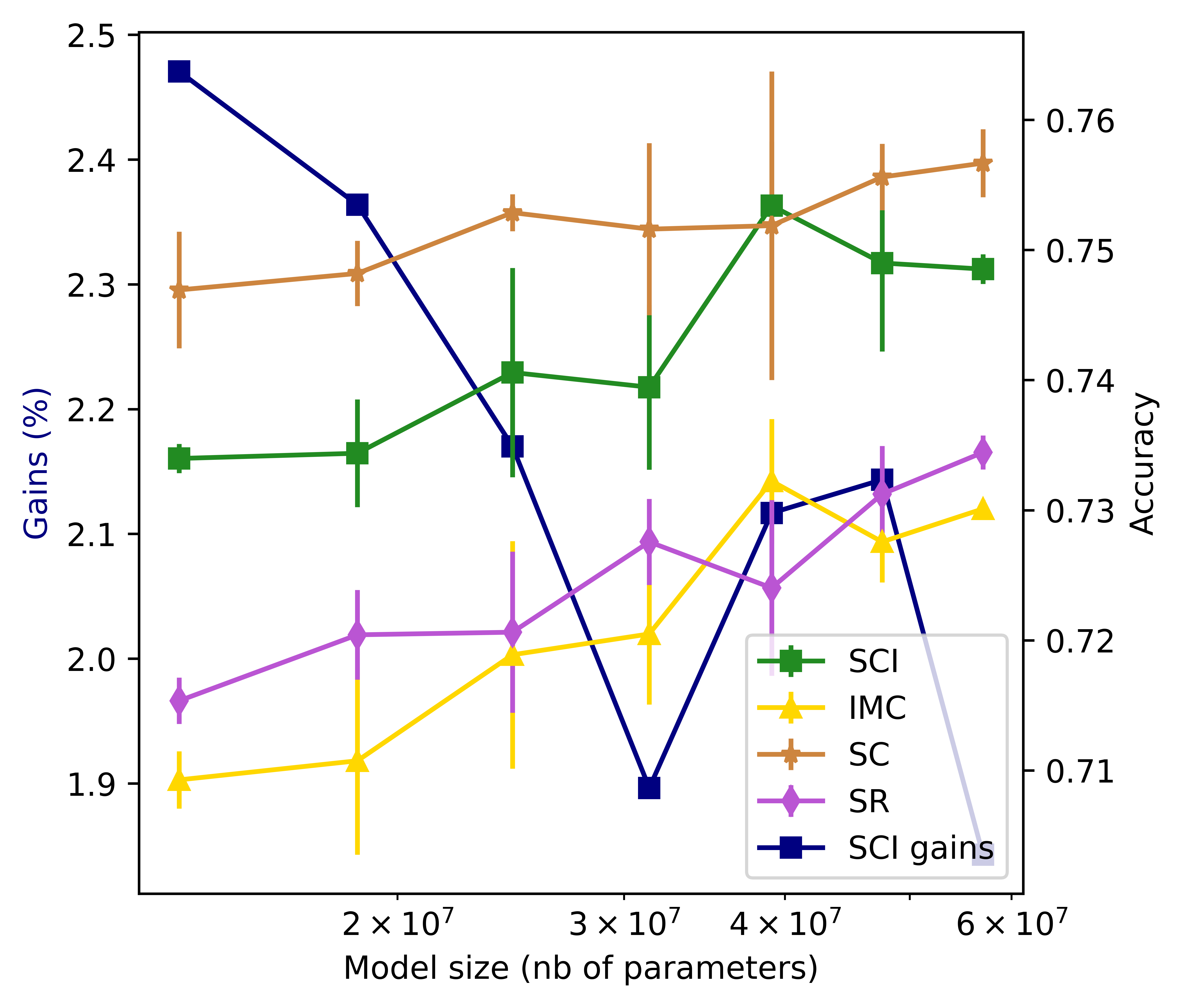}
    \caption{ImageNet after 90 epochs}
    \label{fig:imgnet_modelsize}
\end{subfigure}

\caption{Results at the last epoch on all tasks and for all four techniques (error bars show the standard deviation over several seeds)}
\label{fig:modelsize_results}
\end{figure*}

\subsection{A new multi-scale methodology}
Most papers in the field evaluate the benefits of their NeSy technique compared to other NeSy techniques or a purely neural baseline on very few performance points, usually obtained on 2-3 base networks trained on the full training set for a fixed amount of epochs. A technique will be judged relevant if it increases performance significantly in all scenarios. Although informative, such performance-driven methodology paints a very limited picture of the benefits of the technique and leaves many questions unanswered. In particular, it does not allow to estimate how these benefits evolve when resources given to the system (\eg network scale, dataset size, training length, etc.) increase. Besides, it is well known that performance is marginally increasingly more expensive to obtain in terms of resources \cite{Rosenfeld2019}, and appreciation of the value of background knowledge and NeSy techniques must take that into account. This leads to two main shortcomings. First, it provides limited insights into the sustainability of these benefits as the trend toward larger and deeper networks unfolds.  Secondly, answering frugality-driven questions through this methodology is impossible.

To overcome those limitations, we developed a multi-scale evaluation methodology that studies the dependency between the performance of the NeSy technique and the resources of the system. For each task, we selected a single architectural design that can be scaled to various sizes. Then, across networks of different scales, we compared neural-based classification systems implementing 4 different techniques : 1 purely neural (\textit{imc}) and 3 neurosymbolic (\textit{sr}, \textit{sc} and \textit{sci}). We trained each system on the training set for up to 100 epochs while evaluating the \textbf{exact accuracy} (\ie the share of instances which are well classified on all labels) on the test set at each epoch. We reproduced each experience with different seeds and aggregated the results to measure the variability and statistical significance of our results. More details on the experimental setup are given in the supplementary materials.

Moreover, this methodology allows us to model the results points recorded for each technique \cite{Rosenfeld2019} and therefore compute interesting metrics, such as asymptotic accuracy or frugality-driven like compression ratios (see the supplementary materials for more information).

\subsection{Tasks}

\subsubsection{Categorical classification} \label{sec:MNIST}
MNIST is one of the oldest and most popular dataset in computer vision and consists of small images of hand-written digits (\eg \inlinegraphics{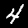} or \inlinegraphics{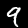}). Since its introduction in \cite{LeCun1998}, it has been used as a \textit{toy} dataset in many different settings. Likewise in neurosymbolic literature, many researchers used MNIST as a basis to build structured dataset compositionally (\eg the PAIRS dataset in \cite{Marra2020}, the MNIST-Add dataset in \cite{Manhaeve2021,Badreddine2022,Krieken2022} or the Sudoku dataset in \cite{Augustine2022VisualSP}).

We mentioned earlier (see Section \ref{sec:formalism}) how categorical classification tasks could be framed as a multi-label classification with background knowledge. For this dataset, the architectural design is a simple Convolutional Neural Network (CNN) \cite{LeCun1998} scaled to different sizes.


\begin{figure*}[t]
\begin{subfigure}{0.32\textwidth}
    \centering
    \includegraphics[width=\textwidth]{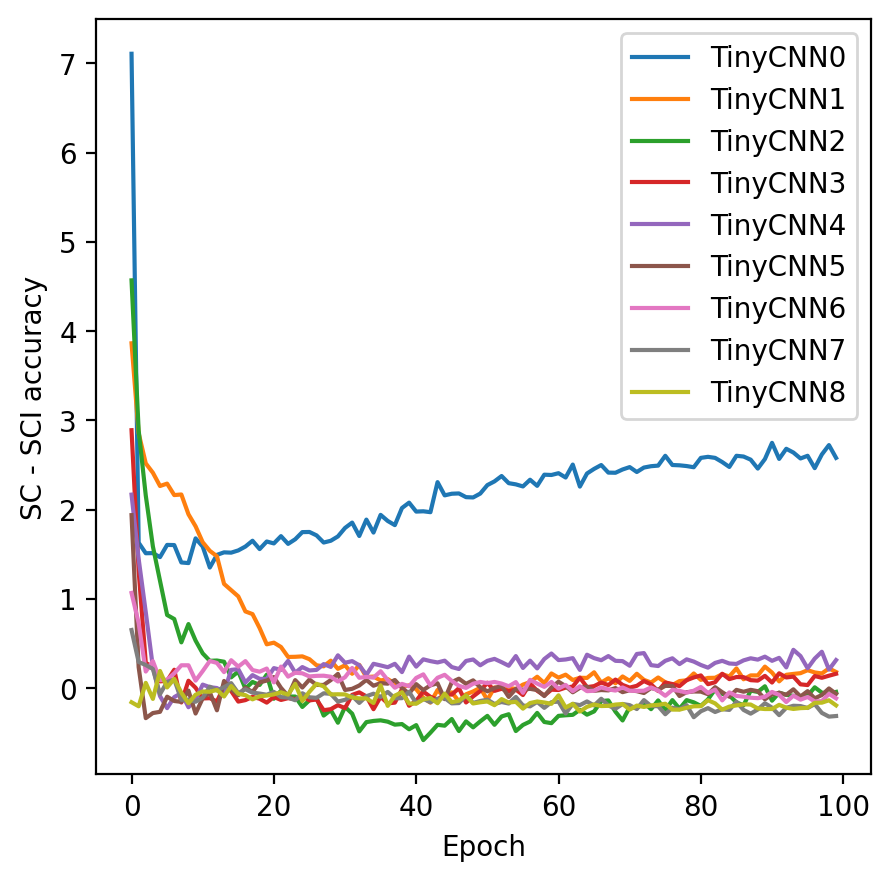}
    \caption{MNIST}
    \label{fig:mnist_delta}
\end{subfigure}
\begin{subfigure}{0.32\textwidth}
    \centering
    \includegraphics[width=\textwidth]{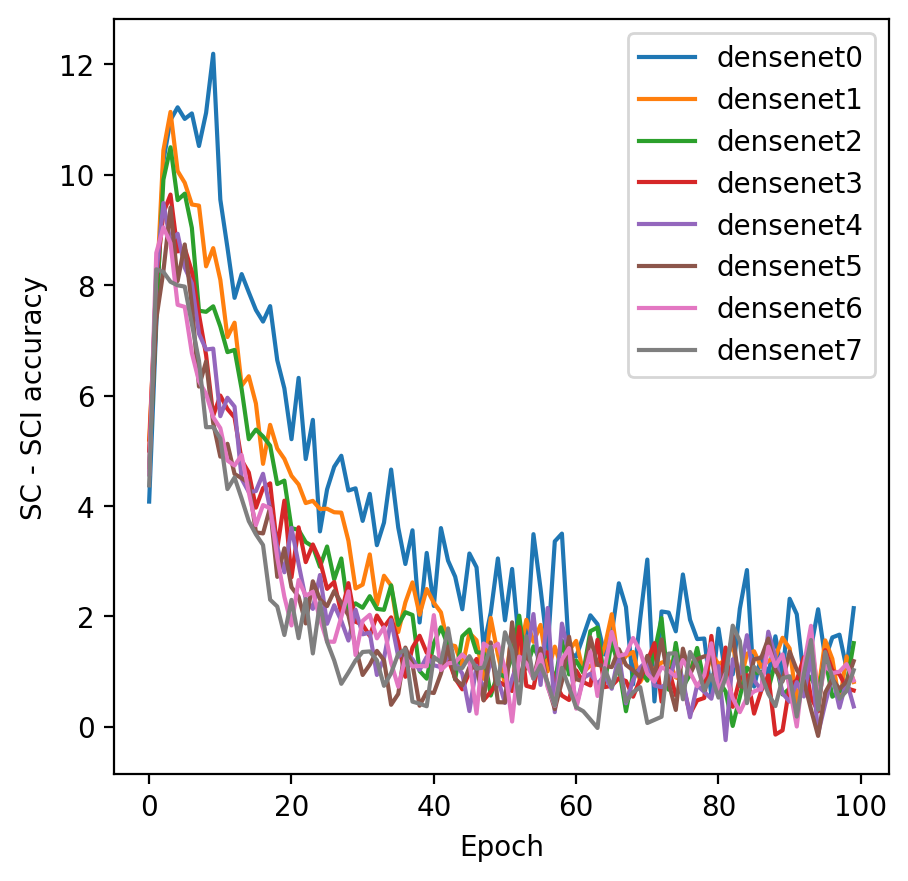}
    \caption{Cifar}
    \label{fig:cifar_delta}
\end{subfigure}
\begin{subfigure}{0.32\textwidth}
    \centering
    \includegraphics[width=\textwidth]{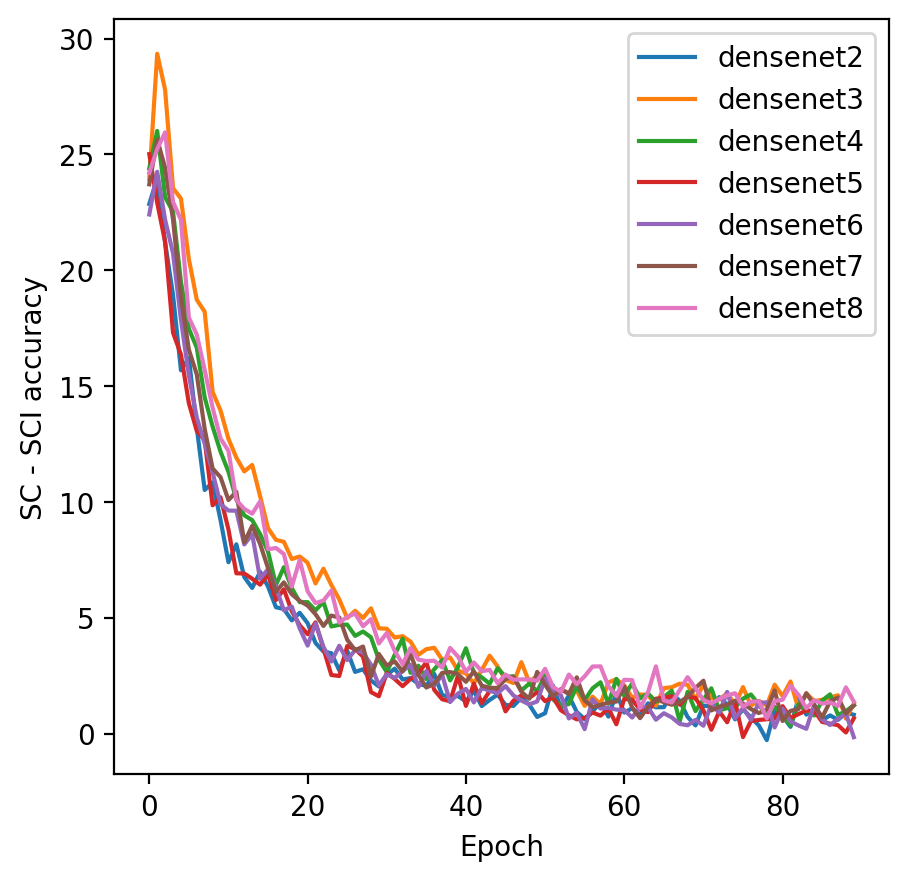}
    \caption{ImageNet}
    \label{fig:imgnet_delta}
\end{subfigure}

\caption{Accuracy gap between \textit{sc} and \textit{sci} for all model sizes on all tasks}
\label{fig:acc_gap}
\end{figure*}

\subsubsection{Hierarchical classification}
The Cifar-100 dataset \cite{Krizhevsky2009} is composed of 60,000 images classified into 20 mutually exclusive super-classes (\eg \textit{reptile}), each divided into 5 mutually exclusive fine-grained classes (\eg \textit{crocodile}, \textit{dinosaur}, \textit{lizard}, \textit{turtle}, and \textit{snake}). While most papers only consider the categorical classification task arising form the 100 fine-grained classes, we keep all 120 classes to produce a multi-label classification task where background knowledge captures mutual exclusion and the hierarchy between super and fine-grained classes.

The ImageNet Large Scale Visual Recognition Challenge (ILSVRC) \cite{Russakovsky2015} is an image classification challenge running annually since 2010 and has become a standard benchmark in computer vision to compare performances of deep learning models. As of August 2014, ImageNet contained 14,197,122 annotated images organized into 21,841 synsets of the WordNet hierarchy \cite{Miller1995}, however standard image classification tasks often use a subset of those, usually 1,000 or 100 synsets. The WordNet hierarchy defines subsumption (or inclusion) between classes, and can be used in many ways to create a task of binary multi-label classification with background knowledge.

For our experiments, we sample 100 classes from 1k ImageNet and add all their parent classes. We then prune classes that have only one parent class and one child class to avoid classes having identical sample sets. We thus obtain a dataset of ImageNet samples labeled on a hierarchy of 271 classes. Background knowledge for this task includes the hierarchical knowledge coming from WordNet, as well as exclusion knowledge that we obtain by assuming two classes having no common descendants are mutually exclusive.


For hierarchical datasets, we used networks of different sizes taken from the family of DenseNets \cite{Huang2017}.

\subsection{Results}
We synthesize our main experimental findings in the three following observations.

\begin{obs}[see Figure \ref{fig:modelsize_results}]
    For all three tasks and across model scales, techniques based on the conditioned inference module $\mathsf{I}_{|\kappa}$ (\eg \textit{sc} and \textit{sci}) \textbf{significantly outperform} those based on $\mathsf{I}_{imc}$ (\eg \textit{imc} and \textit{sr}).
\end{obs}

\begin{obs}[see Figure \ref{fig:modelsize_results}]
    The \textbf{gains tend to decrease with the scale} of the neural network and converge towards a significantly positive value.
\end{obs}

\begin{obs}[see Figure \ref{fig:acc_gap}]
    For all three tasks, \textbf{\textit{sc} trains faster than \textit{sci}}: the gap between \textit{sc} and \textit{sci} reaches its peak in the first epochs of training, and then decreases to a much lower value as both techniques converge. The larger the network, the less significant the gap is.
\end{obs}

\section{Related works} \label{sec:related}

\textbf{Alternative logics} can be used to represent the background knowledge: HEX-graphs in \cite{Deng2014}, Prolog in \cite{Manhaeve2021}, ASP in [NeurASP], grounded First Order Logic in \cite{Badreddine2022}. The trade-off for these representation languages is mainly between expressivity and tractability.

\textbf{Semantic conditioning} techniques use probabilistic logics to condition the probability distribution on the output space produced by the neural model on the background knowledge. This impacts both the loss and the inference modules. \cite{Deng2014} introduced this technique for hierarchical knowledge expressed as a HEX-graph, \cite{Ahmed2022spl} generalized the technique to propositional knowledge and \cite{Manhaeve2021} use a similar technique on ProbLog programs \cite{DeRaedt07}.

\textbf{Semantic regularization} can be understood more broadly as a class of techniques that turn logical formulas into regularization terms which are added to the standard cross-entropy loss to steer their neural models towards outputs that satisfy the background knowledge, without affecting inference. They mainly differ in the logics used to represent the background knowledge and transform it into a loss function: \cite{Xu2018} uses exact Probabilistic Propositional Logic, \cite{Badreddine2022} uses Fuzzy First Order Logic and \cite{pmlr-v206-ahmed23a} approximates Probabilistic Propositional Logic to tackle formulas which are intractable for exact methods.

\textbf{Fuzzy Logics} is sometimes seen as an alternative when the background knowledge becomes intractable to exact Probabilistic Logics. Fuzzy logics are determined by the choice of fuzzy operators, which are relaxations of binary operators $\neg, \land, \lor, \Rightarrow$. \cite{Marra2019_tnorms} shows how a set of fuzzy operators can be unambiguously determined given the selection of a t-norm generator and \cite{Krieken2020} studies the gradient flows of fuzzy operators to determine which are the most amenable to efficient gradient descent.

\textbf{Semi-supervised} learning is used to take advantage of unlabeled labeled data. Semantic regularization is well suited to semi-supervised learning because the regularization term does not depend on the ground truth label. Its efficiency has already been demonstrated in \cite{Xu2018}. By substituting the independent probability distribution in \cite{Grandvalet2004} with a distribution conditioned on background knowledge, \cite{Ahmed2022nesyer} introduces a neurosymbolic entropy regularization term which steers the neural model towards low-entropy (\ie high confidence) states that entail the semantics of the task. This regularization term can work in conjunction to either semantic conditioning or semantic regularization to improve semi-supervised performances.

\textbf{Weakly-supervised} learning tasks allow for partially labeled data. Some semantic conditioning techniques are adapted to this setting and can be used to learn a representation of latent variables, which are not \textit{observed} (\eg labeled) during training, through the relationships they have with other variables that are expressed in the background knowledge. A standard task in that vein is MNIST-Add, which aim is to learn a latent representation of hand-written digits from \textit{observing} their sum.

Finally, we assume throughout the paper that the background knowledge is known \textit{a priori}, which is often not the case in practice. Discovering the \textit{structure} of the task at hand and training the model simultaneously is an important field of research.

\section{Conclusion} \label{sec:conclusion}
In this paper, we introduced a new formalism for supervised multi-label classification with propositional background knowledge, a new neurosymbolic technique to leverage this background knowledge and a new methodology to evaluate the benefits of neurosymbolic techniques across model scales. To the best of our knowledge, we are the first to show that benefits from neurosymbolic techniques do not vanish as the neural network scales. Our results demonstrate that semantic conditioning at inference can significantly improve the performance of a neural-based classification system across network scales, reduce resource needs, while guaranteeing the semantic consistency of outputs.

Future directions for our work may include, amongst others, reproducing our experiments on more datasets (\eg iNaturalist \cite{iNaturalist}, MNIST Sudoku \cite{Augustine2022VisualSP} or Road-R \cite{Giunchiglia2023}), investigating the semi-supervised and weakly-supervised settings, exploring more diverse and expressive logics to represent background knowledge (\eg Answer Set Programming, Relational Logic, First Order Logic, etc.) or automatically extracting the background knowledge from labeled data.

\bibliographystyle{named.bst}
\bibliography{main.bib}

\end{document}


\title{\textit{Supplementary materials to} \\ Improving Neural-based Classification with Logical Background Knowledge}

\author{Anonymous}

\date{}

\maketitle

\section{Neurosymbolic techniques}
In this section, we give proofs for the properties of neurosymbolic techniques presented in the paper.

First, let us demonstrate that standard modules for independent and categorical classification are particular cases of semantic conditioning on their respective background knowledge:
\begin{proposition} \label{prop:sem_cond_equiv}
    \begin{equation}
        \begin{split}
        \mathsf{L}_{|\top}(\mathbf{a}, \mathbf{y}) = \mathsf{L}_{imc}(\mathbf{a}, \mathbf{y}) \quad & \mathsf{I}_{|\top}(\mathbf{a}) = \mathsf{I}_{imc}(\mathbf{a}) \\
        \mathsf{L}_{|\kappa_{\odot_k}}(\mathbf{a}, \mathbf{y}) = \mathsf{L}_{\odot_k}(\mathbf{a}, \mathbf{y}) \quad & \mathsf{I}_{|\kappa_{\odot_k}}(\mathbf{a}) = \mathsf{I}_{\odot_k}(\mathbf{a})
        \end{split}
    \end{equation}
\end{proposition}

We start by demonstrating the following lemma:
\begin{lemma} \label{lemma:exp}
Let's assume $\mathbf{a} \in \mathbb{R}^k$, then:
    \begin{equation*}
    \mathbf{P}(\mathbf{y} | \mathbf{a}) = \prod_{1 \leq j \leq k} y_j.\mathsf{s}(a_j) + (1-y_j).(1-\mathsf{s}(a_j))
    \end{equation*}
where $\mathsf{s}(\mathbf{a}) = (\frac{e^{a_j}}{1 + e^{a_j}})_{1\leq j \leq k}$ is the sigmoid function. 
\end{lemma}

\begin{proof}
    To prove this, let's prove by recurrence on $k \in \mathbb{N}^*$ that:
    \begin{equation*}
        \forall \mathbf{a} \in \mathbb{R}^k, \mathsf{Z}(\mathbf{E}(\cdot | \mathbf{a})) = \prod_{1 \leq j \leq k} (1 + e^{a_j})
    \end{equation*}
    First, let's assume $k=1$, we have:
    \begin{equation*}
        \forall a \in \mathbb{R}, \mathsf{Z}(\mathbf{E}(\cdot | a)) = \mathbf{E}(0 | a) + \mathbf{E}(1 | a) = e^0 + e^a = 1 + e^a
    \end{equation*}
    Then, let's assume $k>1$, we have:
    \begin{equation*}
    \begin{split}
        \forall \mathbf{a} \in \mathbb{R}^k, \mathsf{Z}(\mathbf{E}(\cdot | \mathbf{a})) & = \sum_{\mathbf{y} \in \{0, 1\}^k} \mathbf{E}(\mathbf{y} | \mathbf{a}) = \sum_{\mathbf{y} \in \{0, 1\}^k} \prod_{1 \leq i\leq k} e^{a_i.y_i} \\
        & = \sum_{\substack{\mathbf{y} \in \{0, 1\}^k \\ y_k=0}} \prod_{1 \leq i\leq k-1} e^{a_i.y_i} + \sum_{\substack{\mathbf{y} \in \{0, 1\}^k \\ y_k=1}} e^{a_k}. \prod_{1 \leq i\leq k-1} e^{a_i.y_i} \\
        & = (1+e^{a_k}) . \sum_{\mathbf{y} \in \{0, 1\}^{k-1}} \prod_{1 \leq i\leq k-1} e^{a_i.y_i} = (1+e^{a_k}) . \mathsf{Z}(\mathbf{E}(\cdot | \mathbf{a}_{\setminus k}))
    \end{split}
    \end{equation*}
    where $\mathbf{a}_{\setminus k} = (a_j)_{1 \leq j \leq k-1}$.
    
    By application of the recurrence hypothesis:
    \begin{equation*}
        \mathsf{Z}(\mathbf{E}(\cdot | \mathbf{a}_{\setminus k})) = \prod_{1 \leq j \leq k-1} (1 + e^{a_j})
    \end{equation*}
    Hence:
    \begin{equation*}
        \forall \mathbf{a} \in \mathbb{R}^k, \mathsf{Z}(\mathbf{E}(\cdot | \mathbf{a})) = \prod_{1 \leq j \leq k} (1 + e^{a_j})
    \end{equation*}
    This gives us:
    \begin{equation*}
        \forall \mathbf{y} \in \{0, 1\}^k, \forall \mathbf{a} \in \mathbb{R}^k, \mathbf{P}(\mathbf{y} | \mathbf{a}) = \frac{\mathbf{E}(\mathbf{y}|\mathbf{a})}{\mathsf{Z}(\mathbf{E}(\cdot | \mathbf{a}))} = \frac{\prod_{1 \leq i\leq k} e^{a_i.y_i}}{\prod_{1 \leq j \leq k} (1 + e^{a_j})} = \prod_{1 \leq j \leq k} \frac{e^{a_i.y_i}}{1 + e^{a_j}}
    \end{equation*}

    Notice that:
    \begin{equation*}
        \forall y \in \{0, 1\}, a \in \mathbb{R}, \frac{e^{a.y}}{1 + e^a} = y.\mathsf{s}(a) + (1-y).(1-\mathsf{s}(a))
    \end{equation*}
    Thus, finally:
    \begin{equation*}
        \forall \mathbf{y} \in \{0, 1\}^k, \forall \mathbf{a} \in \mathbb{R}^k, \mathbf{P}(\mathbf{y} | \mathbf{a}) = \prod_{1 \leq j \leq k} y_j.\mathsf{s}(a_j) + (1-y_j).(1-\mathsf{s}(a_j))
    \end{equation*}
\end{proof}

\begin{proof}[\proofname\ \ref{prop:sem_cond_equiv}.1]
First, according to Lemma \ref{lemma:exp}:
\begin{equation*}
    \mathbf{P}(\mathbf{y} | \mathbf{a}) = \prod_{1 \leq j \leq k} y_j.\mathsf{s}(a_j) + (1-y_j).(1-\mathsf{s}(a_j))
\end{equation*}
Besides, we know that $\forall \mathbf{y} \in \{0, 1\}^k, \mathbf{y} \models \top$, which implies that:
\begin{equation*}
    \forall \mathbf{y} \in \{0, 1\}^k, \forall \mathbf{a} \in \mathbb{R}^k, \mathbf{P}(\mathbf{y} | \mathbf{a}, \top) = \mathbf{P}(\mathbf{y} | \mathbf{a})
\end{equation*}.
This gives:
\begin{equation*}
\begin{split}
    \mathsf{L}_{|\top}(\mathbf{a}, \mathbf{y}) & = - \log(\mathbf{P}(\mathbf{y} | \mathbf{a}, \top)) = - \log(\mathbf{P}(\mathbf{y} | \mathbf{a})) \\
    & = - \log(\prod_j y_j.\mathsf{s}(a_j) + (1-y_j).(1-\mathsf{s}(a_j))) \\
    & = - \sum_j \log(y_j.\mathsf{s}(a_j) + (1-y_j).(1-\mathsf{s}(a_j))) \\
\end{split}
\end{equation*}
Since $\mathbf{y}$ is a binary vector:
\begin{equation*}
\begin{split}
    \mathsf{L}_{|\top}(\mathbf{a}, \mathbf{y}) & = - \sum_j y_j.\log(\mathsf{s}(a_j)) + (1-y_j).\log(1-\mathsf{s}(a_j)) \\
    & = \mathsf{L}_{imc}(\mathbf{a}, \mathbf{y})
\end{split}
\end{equation*}
\end{proof}

\begin{proof}[\proofname\ \ref{prop:sem_cond_equiv}.2]
    \begin{equation*}
    \begin{split}
    \mathsf{I}_{|\top}(\mathbf{a}) & = \underset{\mathbf{y} \in \{0, 1\}^k}{\argmax}\mathbf{P}(\mathbf{y} | \mathbf{a}, \top)
    = \underset{\mathbf{y} \in \{0, 1\}^k}{\argmax}\mathbf{P}(\mathbf{y} | \mathbf{a})
    = \underset{\mathbf{y} \in \{0, 1\}^k}{\argmax}\mathbf{E}(\mathbf{y} | \mathbf{a}) \\
    & = \underset{\mathbf{y} \in \{0, 1\}^k}{\argmax}\prod_{1 \leq i\leq k} e^{a_i.y_i}
     = \underset{\mathbf{y} \in \{0, 1\}^k}{\argmax}[\exp(\sum_{1 \leq i\leq k} a_i.y_i)]
     = \underset{\mathbf{y} \in \{0, 1\}^k}{\argmax}\sum_{1 \leq i\leq k} a_i.y_i \\
    & = \mathbf{1}[\mathbf{a} \geq 0] \\
    & = \mathsf{I}_{imc}(\mathbf{a}) 
    \end{split}
\end{equation*}
\end{proof}

\begin{proof}[\proofname\ \ref{prop:sem_cond_equiv}.3]
The one and only one semantic of $\kappa_{\odot_k}$ gives us:
\begin{equation*}
    \forall \mathbf{y}, \mathbf{y} \models \kappa_{\odot_k} \implies \exists j, \mathbf{y} = \odot_k(j)
\end{equation*}
Hence:
\begin{equation*}
\begin{split}
     \mathbf{P}(\kappa_{\odot_k} | \mathbf{a}) & = \sum_{\mathbf{y} \models \kappa_{\odot_k}} \mathbf{P}(\mathbf{y} | \mathbf{a}) = \sum_{1 \leq j \leq k} \mathbf{P}(\odot_k(j) | \mathbf{a}) \\
     & = \frac{1}{\mathsf{Z}(\mathbf{P}(\cdot | \mathbf{a}))} . \sum_{1 \leq j \leq k} \mathbf{E}(\odot_k(j) | \mathbf{a})
      = \frac{1}{\mathsf{Z}(\mathbf{P}(\cdot | \mathbf{a}))} . \sum_{1 \leq j \leq k} e^{a_j}
\end{split}
\end{equation*}
This leads to:
\begin{equation*}
\begin{split}
    \forall l, \mathbf{P}(\odot_k(l) | \mathbf{a}, \kappa_{\odot_k}) & = \frac{\mathbf{P}(\odot_k(l) \land \kappa_{\odot_k}| \mathbf{a})}{\mathbf{P}(\kappa_{\odot_k}| \mathbf{a})} = \frac{\mathbf{P}(\odot_k(l) | \mathbf{a})}{\mathbf{P}(\kappa_{\odot_k}| \mathbf{a})} = \frac{\mathbf{E}(\odot_k(l) | \mathbf{a})}{\mathsf{Z}(\mathbf{P}(\cdot | \mathbf{a})) . \mathbf{P}(\kappa_{\odot_k}| \mathbf{a})} \\
    & = \frac{e^{a_l}}{\sum_{1 \leq j \leq k} e^{a_j}} = \sigma(\mathbf{a})_l = \langle \sigma(\mathbf{a}), \odot_k(l) \rangle
\end{split}
\end{equation*}

Besides, since we assume consistent labels, we know that there is $l$ such that $\mathbf{y} = \odot_k(l)$, which gives:
\begin{equation*}
\begin{split}
    \mathsf{L}_{|\kappa_{\odot_k}}(\mathbf{a}, \mathbf{y}) & = \mathsf{L}_{|\kappa_{\odot_k}}(\mathbf{a}, \odot_k(l)) = - \log(\mathbf{P}(\odot_k(l) | \mathbf{a}, \kappa_{\odot_k})) \\
    & = - \log(\langle \sigma(\mathbf{a}), \odot_k(l) \rangle) = - \log(\langle \sigma(\mathbf{a}), \mathbf{y} \rangle) \\
    & = \mathsf{L}_{\odot_k}(\mathbf{a}, \mathbf{y})
\end{split}
\end{equation*}
\end{proof}

\begin{proof}[\proofname\ \ref{prop:sem_cond_equiv}.4]
We know that $\kappa_{\odot_k}$ is satisfiable and $\mathbf{P}(\mathbf{y} | \mathbf{a})$ is strictly positive. So we have:
\begin{equation*}
\begin{split}
        \mathbf{y} = \underset{\mathbf{y} \in \{0, 1\}^k}{\argmax}\mathbf{P}(\mathbf{y} | \mathbf{a}, \kappa_{\odot_k}) & \implies \mathbf{y} \models \kappa_{\odot_k} \\
    & \implies \exists l, \mathbf{y} = \odot_k(l)
\end{split}
\end{equation*}

Therefore, we have:
\begin{equation*}
\begin{split}
    \mathsf{I}_{|\kappa_{\odot_k}}(\mathbf{a}) & = \underset{1 \leq j \leq k}{\argmax}\mathbf{P}(\odot_k(j) | \mathbf{a}, \kappa_{\odot_k}) = \underset{1 \leq l \leq k}{\argmax}\mathbf{P}(\odot_k(l) | \mathbf{a}) \\
    & = \underset{1 \leq l \leq k}{\argmax}\langle \mathbf{a}, \odot_k(l) \rangle = \odot_k(\underset{1 \leq l \leq k}{\argmax}(\mathbf{a})) \\
    & = \mathsf{I}_{\odot_k}(\mathbf{a})
\end{split}
\end{equation*}
\end{proof}

Then, we prove syntactic invariance of semantic conditioning (and, incidentally, of semantic conditioning at inference):
\begin{proposition}[Syntactic invariance] \label{prop:syntactic_invariance}
\begin{gather*}
        \kappa \equiv \gamma \implies \mathsf{L}_{|\kappa} = \mathsf{L}_{|\gamma} \\
        \kappa \equiv \gamma \implies \mathsf{I}_{|\kappa} = \mathsf{I}_{|\gamma}
    \end{gather*}
\end{proposition}

\begin{proof}[\proofname\ \ref{prop:syntactic_invariance}]
    Let $\kappa$ and $\gamma$ be two equivalent propositional formulas.
    We have:
    \begin{equation*}
        \mathbf{y} \in \{0, 1\}^k, \mathbf{y} \models \kappa \iff \mathbf{y} \models \gamma
    \end{equation*}
    For a probability distribution $\mathbf{P}$ this implies:
    \begin{gather*}
        \mathbf{P}(\cdot \land \kappa) = \mathbf{P}(\cdot \land \gamma) \\
        \implies \mathbf{P}(\kappa) = \mathbf{P}(\gamma) \\
        \implies \mathbf{P}(\cdot | \kappa) = \mathbf{P}(\cdot | \gamma)
    \end{gather*}
    Which means:
    \begin{gather*}
        \forall \mathbf{y}, -\log(\mathbf{P}(\mathbf{y} | \kappa)) = -\log(\mathbf{P}(\mathbf{y} | \gamma)) \\
        \underset{\mathbf{y} \in \{0, 1\}^k}{\argmax}\mathbf{P}(\mathbf{y} | \mathbf{a}, \kappa) = \underset{\mathbf{y} \in \{0, 1\}^k}{\argmax}\mathbf{P}(\mathbf{y} | \mathbf{a}, \gamma)
    \end{gather*}
    Finally, we have:
    \begin{gather*}
        \mathsf{L}_{|\kappa} = \mathsf{L}_{|\gamma} \\
        \mathsf{I}_{|\kappa} = \mathsf{I}_{|\gamma}
    \end{gather*}
\end{proof}

We now demonstrate that semantic conditioning and semantic conditioning at inference are both \textbf{consistent} with the background knowledge:
\begin{proposition}[Semantic consistency] \label{prop:sem_consistency}
\begin{equation*}
    \forall \mathbf{a}, \mathsf{I}_{|\kappa}(\mathbf{a}) \models \kappa
\end{equation*}
\end{proposition}

\begin{proof}[\proofname\ \ref{prop:sem_consistency}]  \label{proof:sem_consistency}
We assumed $\kappa$ to be satisfiable and we know that for all $\mathbf{a}$, $\mathbf{P}(\cdot | \mathbf{a})$ is strictly positive. Therefore $\mathbf{P}(\cdot | \mathbf{a}, \kappa)$ is strictly positive and we have:
\begin{equation*}
\begin{split}
    \mathbf{y} = \underset{\mathbf{y} \in \{0, 1\}^k}{\argmax}\mathbf{P}(\mathbf{y} | \mathbf{a}, \kappa) & \implies \mathbf{P}(\mathbf{y} | \mathbf{a}, \kappa) > 0 \implies \mathbf{y} \models \kappa
\end{split}
\end{equation*}
Hence:
\begin{equation*}
\forall \mathbf{a}, \mathsf{I}_{|\kappa}(\mathbf{a}) = \underset{\mathbf{y} \in \{0, 1\}^k}{\argmax}\mathbf{P}(\mathbf{y} | \mathbf{a}, \kappa) \models \kappa
\end{equation*}
\end{proof}

Besides, when performing inference based on identical model modules and learned parameters, \textit{sci} \textbf{guarantees} greater or equal accuracy compared to traditional \textit{imc} inference (\ie if $\mathsf{I}_{imc}$ infers the right labels, then $\mathsf{I}_{|\kappa}$ will also infer the right labels): 
\begin{proposition} \label{prop:acc_guarantee}
\begin{equation*} 
\begin{multlined}
    \forall (x, \mathbf{y}) \in \mathcal{D}, \forall \theta \in \Theta, \\
    \mathsf{I}_{imc}(\mathsf{M}_{\theta}(x))=\mathbf{y} \implies \mathsf{I}_{|\kappa}(\mathsf{M}_{\theta}(x))=\mathbf{y}
\end{multlined}
\end{equation*} 
\end{proposition}

\begin{proof} [\proofname\ \ref{prop:acc_guarantee}] \label{proof:acc_guarantee}
Let's assume that:
\begin{equation*}
    \mathsf{I}_{|\kappa}(\mathbf{a}):= \hat{\mathbf{y}} \ne \mathbf{y}
\end{equation*}

Since both $\hat{\mathbf{y}}$ and $\mathbf{y}$ are consistent with $\kappa$ (which we assume satisfiable), we have:
\begin{equation*}
    \frac{\mathbf{P}(\hat{\mathbf{y}}| \mathbf{a}, \kappa)}{\mathbf{P}(\mathbf{y}| \mathbf{a}, \kappa)} =  \frac{\mathbf{P}(\hat{\mathbf{y}} \land \kappa| \mathbf{a})}{\mathbf{P}(\mathbf{y}  \land \kappa| \mathbf{a})} = \frac{\mathbf{P}(\hat{\mathbf{y}}| \mathbf{a})}{\mathbf{P}(\mathbf{y}| \mathbf{a})}
\end{equation*}

Because $\hat{\mathbf{y}} = \mathsf{I}_{|\kappa}(\mathbf{a}) = \underset{\mathbf{y} \in \{0, 1\}^k}{\argmax}\mathbf{P}(\mathbf{y} | \mathbf{a}, \kappa)$:
\begin{equation*}
    \mathbf{P}(\hat{\mathbf{y}}| \mathbf{a}, \kappa) \geq \mathbf{P}(\mathbf{y}| \mathbf{a}, \kappa)
\end{equation*}

Hence:
\begin{equation*}
    \frac{\mathbf{P}(\hat{\mathbf{y}}| \mathbf{a})}{\mathbf{P}(\mathbf{y}| \mathbf{a})} = \frac{\mathbf{P}(\hat{\mathbf{y}}| \mathbf{a}, \kappa)}{\mathbf{P}(\mathbf{y}| \mathbf{a}, \kappa)} \geq 1
\end{equation*}

Therefore:
\begin{equation*}
    \mathbf{P}(\hat{\mathbf{y}}| \mathbf{a}) \geq \mathbf{P}(\mathbf{y}| \mathbf{a}) \implies \mathbf{y} \ne \underset{\mathbf{y} \in \{0, 1\}^k}{\argmax}\mathbf{P}(\mathbf{y} | \mathbf{a}) = \mathsf{I}_(\mathbf{a})
\end{equation*}

Eventually we have:
\begin{equation*}
\begin{multlined}
    \forall (x, \mathbf{y}) \in \mathcal{D}, \forall \theta \in \Theta, \\
    \mathsf{I}_(\mathsf{M}_{\theta}(x))=\mathbf{y} \implies \mathsf{I}_{|\kappa}(\mathsf{M}_{\theta}(x))=\mathbf{y}
\end{multlined}
\end{equation*}
\end{proof}

Finally, it is interesting to notice that under the consistent label hypothesis:
\begin{proposition}\label{lem:negative_regularization}
\begin{equation}
    \mathsf{L}_{|\kappa}(\mathbf{a}, \mathbf{y}) = - \log(\mathbf{P}(\mathbf{y} | \mathbf{a})) + \log(\mathbf{P}(\kappa | \mathbf{a})) = \mathsf{L}_{r(-1)}(\mathbf{a}, \mathbf{y}) 
\end{equation}
\end{proposition}

\begin{proof}[\proofname\ \ref{lem:negative_regularization}]  
\begin{equation*}
\begin{split}
    \mathsf{L}_{|\kappa}(\mathbf{a}, \mathbf{y}) & = - \log(\mathbf{P}(\mathbf{y} | \mathbf{a}, \kappa)) = - \log(\frac{\mathbf{P}(\mathbf{y} \land \kappa| \mathbf{a})}{\mathbf{P}(\kappa| \mathbf{a})}) \\
    & = - \log(\mathbf{P}(\mathbf{y}  \land \kappa | \mathbf{a})) + \log(\mathbf{P}(\kappa | \mathbf{a})) = - \log(\mathbf{P}(\mathbf{y} | \mathbf{a})) + \log(\mathbf{P}(\kappa | \mathbf{a})) \\
    & = \mathsf{L}_{r(-1)}(\mathbf{a}, \mathbf{y})
\end{split}
\end{equation*}
\end{proof}

Thus, the loss module of semantic conditioning corresponds to that of semantic regularization with a $\lambda = -1$. Although it seems counter-intuitive that two systems trying to reach the same goal end up using "opposite regularization terms" in their loss module, this is justified by the different inference modules used in each system.

Hence, an implementation for $\mathsf{L}_{r(\lambda)}$ can be used for $\mathsf{L}_{|\kappa}$. Besides, by training systems with regularized loss modules with different $\lambda$ and evaluating with both $\mathsf{I}_{|\kappa}$ and $\mathsf{I}_{imc}$, we can span the entire spectrum of techniques in \textit{imc}, \textit{sr}, \textit{sc} and \textit{sci}.



\section{Experimental details}
\subsection{Implementation}
\subsubsection{Categorical classification}
The architectural design for categorical classification on MNIST is a simple Convolutional Neural Network (CNN) \cite{LeCun1998}, as shown on Listing 1. We trained networks with \texttt{num\_layers} from 1 up to 9 layers.
\begin{lstlisting}[language=Python, escapeinside={($}{$)}, caption=Our TinyNet architecture (PyTorch implementation)]
class TinyCNNs(nn.Module):
    def __init__(self, num_classes: int = 10,
                        num_layers: int = 1,
                        in_channels=1):
        super().__init__()
        convs = []
        for i in range(num_layers):
            convs.append(nn.Conv2d(2**(i//2)*in_channels,
                                    2**((i+1)//2)*in_channels,
                                    5,
                                    padding=2))
            convs.append(nn.ReLU())
        self.convs = nn.Sequential(*convs)
        # self.AdaptativeScale = int(5*2**(num_layers/2))
        self.pool = nn.AdaptiveAvgPool2d(5)
        self.fc = nn.Linear(25*2**(num_layers//2)*in_channels,
                            num_classes)

    def forward(self, x):
        x = self.convs(x)
        x = F.relu(x)
        x = self.pool(x)
        x = torch.flatten(x, start_dim=1)
        x = self.fc(x)
        return x
\end{lstlisting}

Then, to complete the neural based system: we use the loss module shown on Listing 2 (with varying $\lambda$) and inference modules shown in Listings 2 and 3 for \textit{imc} and \textit{}sci respectively.

\begin{lstlisting}[language=Python, caption=\textit{rsc} loss]
scores = self.model(x)
energies = torch.gather(self.scores, 1, y.unsqueeze(dim=1))
log_z = torch.sum(torch.log(torch.exp(self.scores).add(1)),
                            dim=1, keepdim=True)
mc_log_z = torch.log(torch.sum(torch.exp(self.scores),
                                dim=1, keepdim=True))

loss = torch.mean(torch.sub(torch.add(log_z.mul(1+self.lambda),
                                    mc_log_z.mul(-self.lambda)),
                            energies))

return loss
\end{lstlisting}

\begin{lstlisting}[language=Python, caption=\textit{imc} inference]
scores = self.model(x)
return torch.gt(scores, 0)
\end{lstlisting}

\begin{lstlisting}[language=Python, caption=\textit{sci} inference]
scores = self.model(x)
_, idx_max = torch.max(self.scores, dim=1)
return F.one_hot(idx_max, num_classes=scores.shape[1])
\end{lstlisting}

\subsubsection{Hierarchical classification}
The architectural design for hierarchical classification tasks Cifar and ImageNet was based on DenseNets \cite{Huang2017}. We used the \texttt{torchvision} implementation with a naive scaling strategy to create DenseNets of various \texttt{size}, as shown on Listing 5. We trained network with \texttt{size} from 1 up to 8.

\begin{lstlisting}[language=Python, caption=DenseNet scaling]
from torchvision.models.densenet import _densenet

network = _densenet(growth_rate=32,
                    block_config=(6, 12, (size+3)*8,(size+1)*8),
                    num_init_features=64,
                    weights=None,
                    progress=True)
\end{lstlisting}

For the loss and inference modules, we followed \cite{Deng2014} and expressed the hierarchical and exclusion relations as a HEX-graph $H$, then compiles this HEX-graph into a HEX-layer that can compute $\mathsf{I}_{\kappa_H}$ using a sparse max-product message passing algorithm with Viterbi decoding (see \texttt{fastHEXLayer.py}).

To do so, for ImageNet, we first extract hierarchical links from the \texttt{wn\_hyp.pl} file and arrange them into a directed graph. Then, for our experiments, only 100 leaf nodes are randomly sampled from the total 1,000 and the directed graph is trimmed of any node not connected to a sampled leaf node. We also prune nodes that only have one parent and one child to avoid the case where two nodes have the same set of labeled samples (which would make them indistinguishable for the network). This directed graph is fed into a \texttt{HEXGraph} object, which computes the sparse and dense version of the hierarchical and exclusion matrices, builds the corresponding junction tree (using the \texttt{JunctionTree} object) and records the valid states of each clique and the sum-product matrix of the junction tree. The results of these compilations steps can be saved and loaded: the specific files used in this experiment are \texttt{./ImageNet/compilations/100p*}. During training, this \texttt{HEXGraph} is loaded from compilation files and passed on to the \texttt{HEXLayer} which contains the methods to perform \textit{sci}. The code to perform those compilation steps is found in \texttt{ImageNetProcessing.ipynb}.

For Cifar-100, the hierarchy has only two levels (macro and fine-grained classes) and can be retrieved directly online and fed to the \texttt{HEXGraph} object.

\subsection{Metrics}
There are plenty of metrics that can be used to evaluate a classification system.

Simple accuracy averages how many classes were correctly labeled on each sample, however, since multi-label classification datasets with background knowledge are often highly unbalanced (far more negative classes that positive ones) it is often unfit to the task. Precision, recall and f1-score metrics can help tackle with this issue, but they lose track of the semantics of the task.

Semantic consistency counts how many outputs are consistent with the background knowledge. Since \textit{sci} is provably consistent, this metric is of little interest for us.

Metrics that are not based on the binary outputs but need to access probability scores associated with each classe, like threshold metrics (\eg map@50, map@75, auc) or top-k metrics, are not accessible to our classification system as is.

Eventually, we decide to use exact accuracy, which counts how many samples are perfectly labeled: this is the most demanding metric since a single mistake disqualifies the whole output.

\subsection{Hyperparameters}

\paragraph{Optimizer} We use Adam with a learning rate of $10^{-4}$ for all tasks.

\paragraph{Seeds} We set seeds manually with \texttt{torch.manual\_seed(args.seed)}. We used $6$ seeds ($[0, 1, 2, 3, 4, 5]$) for MNIST, $3$ seeds ($[0, 1, 2]$) for Cifar and ImageNet.

\paragraph{Batch size} We use a batch size of 8 for MNIST and Cifar, and increase to 64 for ImageNet to speed up training.

\section{Modeling}
To model the relation between the collective accuracy of the system and its number of parameters or training samples, we use the following functional form:
\begin{equation}\label{eq:model_eq}
    a(m) = \alpha.m^{-b} + a_{\infty}
\end{equation}

It is documented in \cite{Hestness2017DeepLS,Rosenfeld2019,Kaplan2020ScalingLF,Hernandez2021ScalingLF,ghorbani_scaling_2021} that exponential models like in Equation \ref{eq:model_eq} provides a good surrogate model for the performance of a neural network as a function of its size, and that it extrapolates well to large networks even when fitted on small networks.

We fit surrogate models of this kind on the accuracy points of all techniques and use them to quantitatively estimate performance and frugality benefits. To fit the curves of surrogate models we used the \texttt{curve\_fit} function from the \texttt{scipy.optimize} library, which optimizes parameters of the function to converge towards the least squares on given data points. We illustrate this on MNIST and Cifar on Figure \ref{fig:surrogate}.

\begin{figure}[h]
\begin{subfigure}{0.49\textwidth}
    \centering
    \includegraphics[width=\textwidth]{results/MNIST/mnist_modelsize@100wSurrogate.png}
    \caption{MNIST}
\end{subfigure}
\begin{subfigure}{0.49\textwidth}
    \centering
    \includegraphics[width=\textwidth]{results/cifar/cifar_densenets_modelsize@100wSurrogate.png}
    \caption{Cifar}
\end{subfigure}

\caption{Fitting surrogate models on accuracy points}
\label{fig:surrogate}
\end{figure}

Table \ref{tab:ainf} displays the asymptotic accuracies of surrogate models for each technique on both MNIST and Cifar. Performance benefits of neurosymbolic techniques are illustrated with the gain curve, which represents the difference between both surrogate models. The \textbf{asymptotic gain} (\ie the limit of the gain curve when the size of the model increases to infinity) is the difference between asymptotic accuracies of the considered neurosymbolic technique and \textit{imc}. This can be used as a multi-scale metric to summarize the performance benefits of a neurosymbolic technique.

\begin{table}[h]
    \centering
    \begin{tabular}{|c||c|c|c|c|}
        \hline
        Technique & \textit{imc} & \textit{sr} & \textit{sci} & \textit{sc} \\
        \hline
        MNIST & 97.7 & 97.8 & 99.0 & 99.0 \\
        \hline
        Cifar & 68.4 & 69.5 & 70.5 & 71.8 \\
        \hline
    \end{tabular}
    \caption{Asymptotic accuracies}
    \label{tab:ainf}
\end{table}

To illustrate the potential of \textit{sci} to build \textbf{frugal} systems, we plot the curves of resource savings $\epsilon(m)$ and compression ratio $\tau(m)$, which respectively measure how many parameters of our neural model could we spare (while maintaining the same level of performance) by using \textit{sci} instead of \textit{imc}, and how much smaller such a model (or training set) would be, \ie:
\begin{gather}\label{eq:params_gains}
    \epsilon(m) = m - a_{|\kappa}^{-1}(a_{|imc}(m)) \\
    \tau(m) = \frac{\epsilon(m)}{m}
\end{gather}

Both curves are represented for MNIST and Cifar on Figure \ref{fig:savings}. They show an increasing compression ratio on both tasks, meaning that \textit{sci} becomes \textbf{increasingly valuable} (in terms of frugality) as the network scales.

\begin{figure}[h]
\begin{subfigure}{0.49\textwidth}
    \centering
    \includegraphics[width=\textwidth]{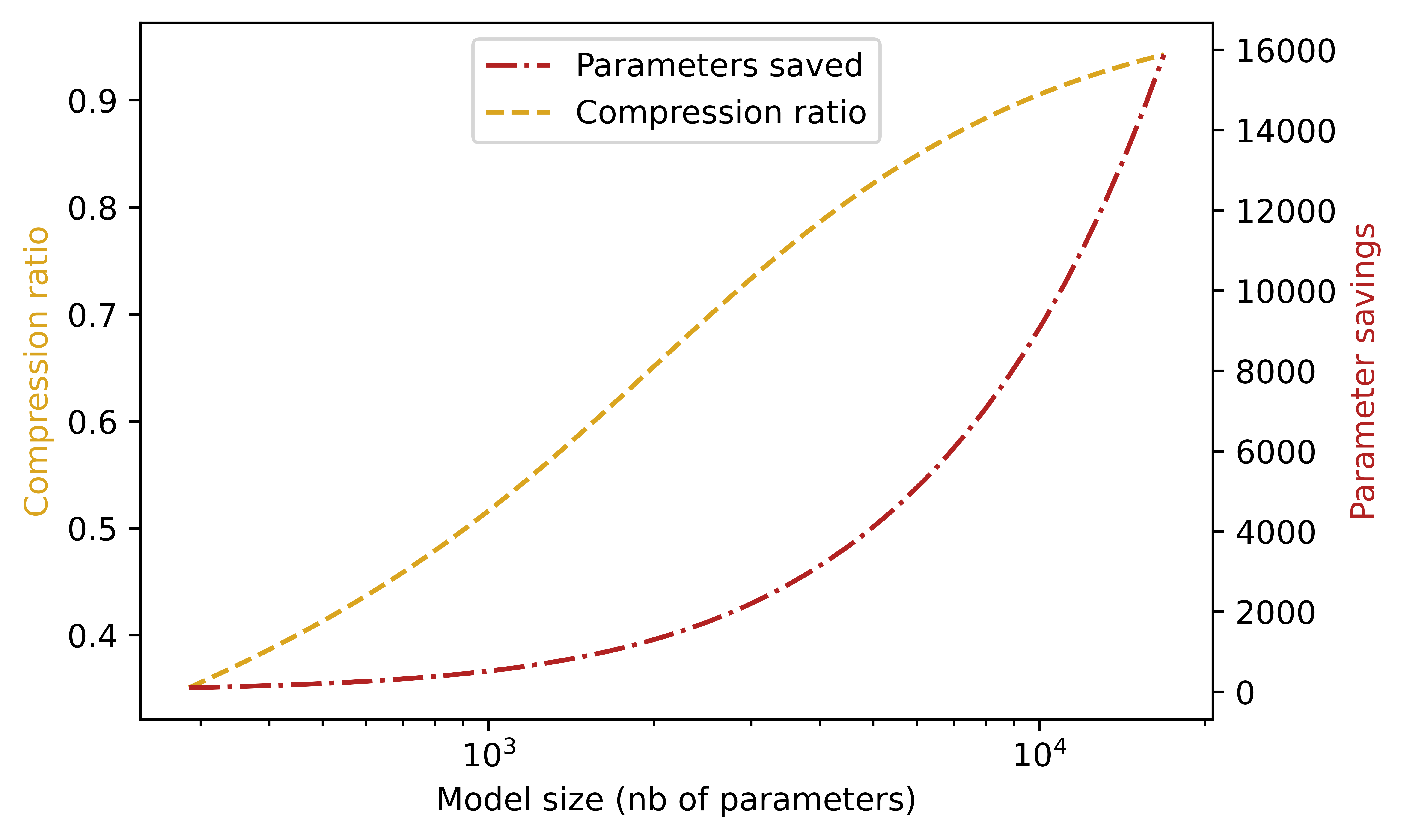}
    \caption{MNIST}
\end{subfigure}
\begin{subfigure}{0.49\textwidth}
    \centering
    \includegraphics[width=\textwidth]{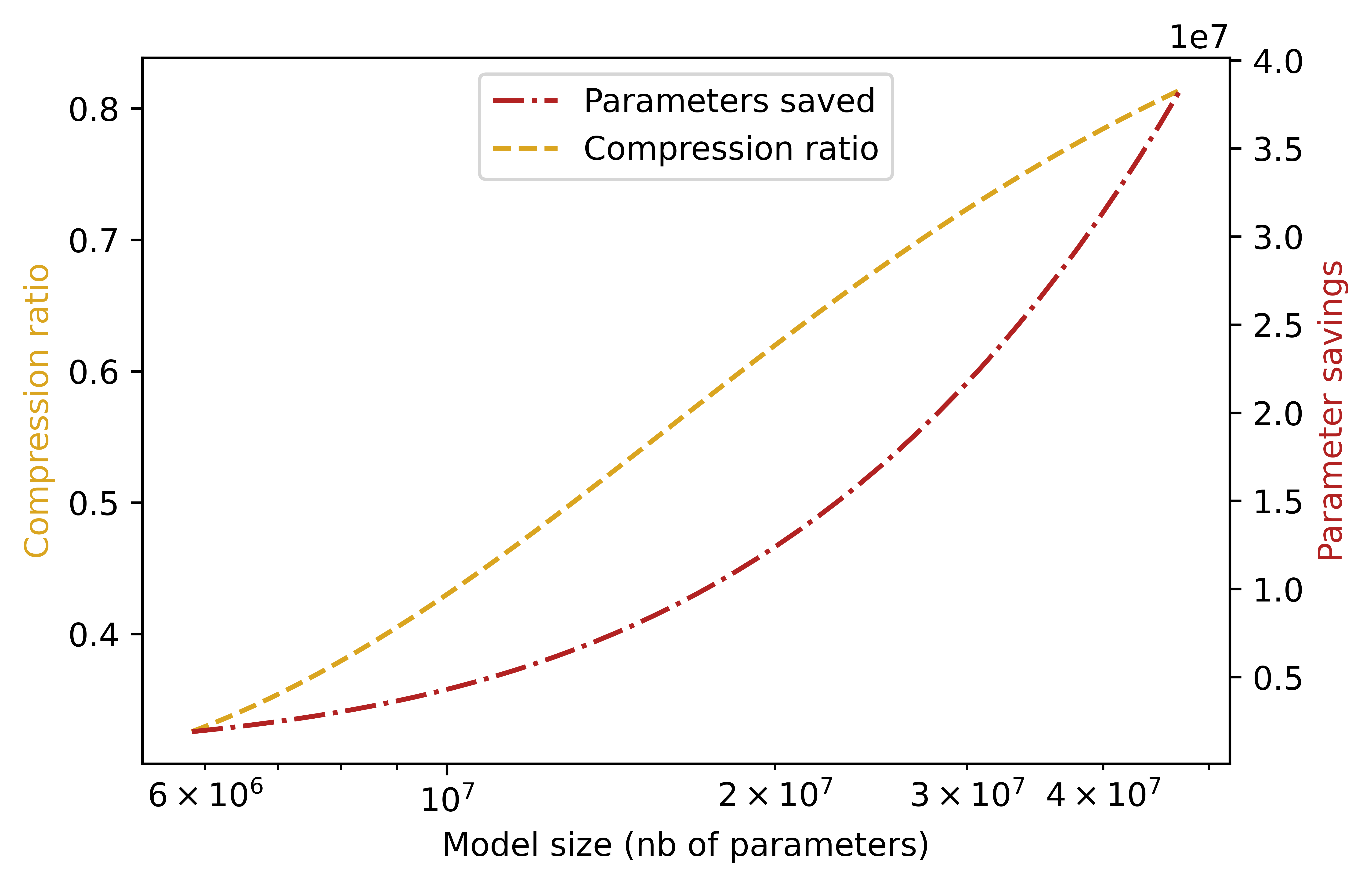}
    \caption{Cifar}
\end{subfigure}

\caption{Parameter savings and model compression potential of \textit{sci}}
\label{fig:savings}
\end{figure}

Some would argue that given the functional form used in Equation \ref{eq:model_eq} and asymptotic gains displayed in Table \ref{tab:ainf} it is expected to get infinite compression. Indeed, a strictly positive asymptotic gain while using exponential models implies that the compression ratio will diverge towards infinity as the network scales towards infinity. However, what Figure \ref{fig:savings} shows is that the compression ratio is monotonically increasing and that it is already quite high in the region of typical network sizes. Both these facts are not necessarily implied by a strictly positive asymptotic gain. Moreover, it is theoretically possible to observe very high compression ratios with insignificant asymptotic gains. All these observations point towards the fact that the analysis of frugality-driven benefits is not redundant with that of performance-driven benefits.

\clearpage
\printbibliography